\documentclass[10pt]{article} 

\usepackage[accepted]{rlj}           

\usepackage{amssymb}            
\usepackage{mathtools}          
\usepackage{mathrsfs}           
\usepackage{graphicx}           
\usepackage{subcaption}         
\usepackage[space]{grffile}     
\usepackage{url}                
\usepackage{lipsum}             

\usepackage{amsthm}
\newtheorem{theorem}{Theorem}
\newtheorem{definition}[theorem]{Definition}
\newtheorem{proposition}[theorem]{Proposition}

\usepackage[ruled,linesnumbered]{algorithm2e}
\usepackage{booktabs}

\newcommand{\bfM}{\mathbf{M}}
\newcommand{\bbM}{\mathbb{M}}
\newcommand{\tbfM}{\widetilde\bfM}
\newcommand{\tbbM}{\widetilde\bbM}

\title{The Confusing Instance Principle for Online Linear Quadratic Control}

\setrunningtitle{Confusing Instance Principle for Linear Quadratic Control}

\author{Waris Radji\textsuperscript{1}, Odalric-Ambrym Maillard\textsuperscript{1}}

\emails{\{waris.radji,odalric.maillard\}@inria.fr}

\affiliations{
$^{1}$\textbf{Univ. Lille, Inria, CNRS, Centrale Lille, UMR 9198-CRIStAL, F-59000 Lille, France}\\
}

\contribution{
    We formulate the Confusing Instance (CI) principle as an optimization problem in the LQR setting, extending this concept beyond MABs and discrete MDPs for the first time.
    }
    {
    The CI principle has previously been applied only in discrete settings, primarily in multi-armed bandits and tabular MDPs \citep{honda2010asymptotically,honda2015non,pesquerel2022imed,balagopalan2024minimum}.
    }

\contribution{
    We develop \texttt{MED-LQ}, a novel control strategy that implements the Minimum Empirical Divergence (MED) framework for online LQR, and show his numerical competitiveness.
    }
    {
    Prior work established MED algorithms in discrete MDPs settings, with \texttt{IMED-RL} for ergodic case \citep{pesquerel2022imed} and \texttt{IMED-KD} for the communicating case \citep{saber2024logarithmic}.
    }

\contribution{
    We develop a novel computational approach for building confusing instances in continuous systems through sensitivity analysis of rank-one perturbations.   
    }
    {
    Prior work limited confusing instances to discrete settings and linear bandits. Our sensitivity analysis for continuous control systems represents the first extension of this principle to linear dynamical systems.
    }

\contribution{
We introduce \texttt{linquax}, a library for efficient research in online LQR problems, built with JAX to leverage automatic differentiation and provide GPU/TPU compatibility.
}
{
Prior to our work, no \emph{modern} open-source library existed specifically for online LQR, creating a significant barrier to reproducible research in this domain.
}


\summary{
We revisit the problem of controlling linear systems with quadratic cost under unknown dynamics within model-based reinforcement learning. Traditional methods like Optimism in the Face of Uncertainty and Thompson Sampling, rooted in multi-armed bandits (MABs), face practical limitations. In contrast, we propose an alternative based on the \textit{Confusing Instance} (CI) principle, which underpins regret lower bounds in MABs and discrete Markov Decision Processes (MDPs) and is central to the \textit{Minimum Empirical Divergence} (MED) family of algorithms, known for their asymptotic optimality in various settings. By leveraging the structure of LQR policies along with sensitivity and stability analysis, we develop \texttt{MED-LQ}. This novel control strategy extends CI and MED principles beyond small-scale settings.

Our work addresses a crucial research gap by exploring whether the CI principle can improve exploration strategies in continuous MDPs. While the exploration-exploitation dilemma is well understood in discrete settings, the curse of dimensionality makes this challenge significantly harder in continuous spaces. \texttt{MED-LQ} overcomes these challenges by efficiently searching for confusing instances through rank-one and entry-wise perturbations while avoiding intractable confidence bounds. Benchmarks on a comprehensive control suite demonstrate that \texttt{MED-LQ} achieves competitive performance across various scenarios, establishing foundations for a fresh perspective on exploration in continuous MDPs and opening new avenues for structured exploration in complex control problems.
}

\begin{document}

\maketitle  

\begin{abstract}
We revisit the problem of controlling linear systems with quadratic cost under unknown dynamics with model-based reinforcement learning. Traditional methods like Optimism in the Face of Uncertainty and Thompson Sampling, rooted in multi-armed bandits (MABs), face practical limitations. In contrast, we propose an alternative based on the \textit{Confusing Instance} (CI) principle, which underpins regret lower bounds in MABs and discrete Markov Decision Processes (MDPs) and is central to the \textit{Minimum Empirical Divergence} (MED) family of algorithms, known for their asymptotic optimality in various settings. By leveraging the structure of LQR policies along with sensitivity and stability analysis, we develop \texttt{MED-LQ}. This novel control strategy extends the principles of CI and MED beyond small-scale settings. Our benchmarks on a comprehensive control suite demonstrate that \texttt{MED-LQ} achieves competitive performance in various scenarios while highlighting its potential for broader applications in large-scale MDPs.
\end{abstract}



\section{Introduction}

In Reinforcement Learning (RL), the exploration-exploitation dilemma is well understood in small-scale settings like multi-armed bandits (MABs) and discrete Markov Decision Processes (MDPs), for which strong theoretical guarantees exist. The curse of dimensionality impacts this dilemma in continuous or high-dimensional spaces, where analyzing this trade-off becomes significantly harder, and traditional exploration strategies struggle to scale. This is evident in deep RL, which, despite its empirical success, e.g. \cite{osband2016deep,bellemare2016unifying,burda2018exploration,sekar2020planning,ladosz2022exploration}, often lacks theoretical foundations.  In this work, we study the exploration-exploitation dilemma in the online \emph{Linear Quadratic Regulator} (LQR) problem where dynamics are \textit{unknown}, in the same setting of \cite{abbasi2011regret}. Widely used in control applications such as robotics and autonomous systems, LQR enables explicit analysis in continuous, structured MDPs \citep{cohen2018online,tu2018least,tu2019gap,maran2025local}.
\paragraph{Research gap.}  Traditional exploration strategies, such as \emph{Optimism in the Face of Uncertainty} (OFU), have been widely applied to LQR and beyond, providing upper regret bounds that evaluate the worst-case performance of a learner, typically scaling as $\widetilde O(\sqrt T)$, but suffer from inherent limitations \citep{lattimore2017end}. On the other hand, lower regret bounds establish fundamental performance limits for any learner on a given problem instance. A key tool in deriving these bounds is the \emph{Confusing Instance} (CI) principle, which constructs hard-to-distinguish problem instances that directly appear in regret lower bound analysis. The \emph{Minimum Empirical Divergence} (MED) family of algorithms is explicitly designed to match these regret lower bounds, leveraging the CI principle to guide exploration efficiently. MED-based methods achieve asymptotic and instance-dependent optimality, often outperforming numerically OFU-based approaches in various settings. Although characterizing regret lower bounds beyond discrete MDPs remains an open research problem and not in the scope of this work, we provide empirical evidence to address the following question,
\begin{center}
\emph{Can the Confusing Instance principle improve exploration strategies in continuous MDPs?}
\end{center}
To the best of our knowledge, this paper is the first to explore the potential of the CI principle in continuous MDPs, through the online LQR setting as an entry point which presents both simplifications and challenges. This work paves the way for novel exploration strategies in large spaces.

\paragraph{From MABs to large MDPs.} RL exploration strategies generally follow a similar evolution. Initially, an idea emerges in discrete MABs. This idea is then extended to linear bandits in parallel with discrete MDPs. The concepts are then applied to continuous MDPs, typically in the LQR setting. Finally, heuristics are developed to tackle high-dimensional problems in deep RL. The evolution of the OFU principle begins with the Upper Confidence Bounds (\texttt{UCB}) algorithm in MABs \citep{auer2002finite}, followed by \texttt{OFUL} in the linear case \citep{abbasi2011improved}. It then extends to \texttt{UCRL} in discrete MDPs \citep{auer2006logarithmic,auer2008near,bourel2020tightening}, \texttt{OFULQ} in LQR \citep{abbasi2011regret,abeille2020efficient,lale2022reinforcement,mete2022augmented}, and finally, in deep RL \citep{bellemare2016unifying,curi2020efficient}. Thompson Sampling (\texttt{TS}) emerged as a more efficient alternative to the OFU principle, relying implicitly on confidence bounds, allowing for analysis similar to OFU. It started with MABs \citep{thompson1933likelihood,kaufmann2012thompson}, then extended to linear MABs with \texttt{LinTS} \citep{agrawal2013thompson,abeille2017linear}, discrete MDPs with \texttt{PSRL} \citep{osband2013more,osband2017posterior}, in LQR \citep{abeille2017thompson,abeille2018improved,pmlr-v178-kargin22a}, and finally to deep RL through Bayesian or ensemble neural networks \citep{osband2016deep,azizzadenesheli2018efficient}. The MED principle\footnote{\cite{baudry2023general} shows that \texttt{MED} and \texttt{TS} can be analyzed following a common methodology.} has seen more recent developments, with its foundation rooted in the regret lower bounds introduced by \cite{lai1985asymptotically} and \cite{burnetas1996optimal,burnetas1997optimal}. First proposed by \cite{honda2010asymptotically,honda2011asymptotically}, the MED principle has been applied to various MABs settings \citep{honda2015non,saber2021indexed,pesquerel2021stochastic,bian2022maillard,saber2024bandits}, and to linear MABs \citep{bian2024indexed,balagopalan2024minimum}. In discrete MDPs, \texttt{IMED-RL} \citep{pesquerel2022imed} emerges as a state-of-the-art algorithm under ergodic assumptions. In communicating MDPs, novel promising strategies explore the MED principle but face the NP-hard challenge of finding CIs \citep{saber2024logarithmic,boone2025regret}. In our paper, we propose to continue the evolution of MED by extending it beyond MABs and discrete MDPs.

\paragraph{Outline and contributions.} Our paper makes several key contributions to RL for unknown LQ systems. After formalizing the problem setup in Section \ref{section-2}, we present a novel formulation of CIs as an optimization problem in Section \ref{section-3}, developing an efficient solution method specifically engineered for LQR. Section \ref{section-4} introduces our main algorithmic contribution, \texttt{MED-LQ}, which leverages these CIs to enable principled exploration while maintaining computational tractability through careful sensitivity and stability analysis. In Section \ref{section-5}, we present comprehensive empirical evaluations across both classical control benchmarks and industrial applications, demonstrating that \texttt{MED-LQ} matches state-of-the-art performance while overcoming the practical limitations of OFU approaches. Our work bridges an important gap between theoretical optimality and practical implementation in continuous control settings, with broader implications for exploration in large-scale MDPs.







\section{Setup and Background material} \label{section-2}

\paragraph{The optimal control problem.} 
Consider a linear time-invariant system written in state-space form, where the state $x_t\in\mathbb R^d$ evolves according to the discrete-time dynamics
\citep{bertsekas2012dynamic}
\begin{equation} \label{dynamics}
    x_{t+1} = A x_t + B u_t + w_t,
\end{equation}
upon receiving control $u_t\in\mathbb R^k$,
where the system matrices $A \in \mathbb R^{d \times d}$ and $B \in \mathbb R^{d \times k}$ govern  the dynamics of the system, and \( w_t \sim \mathcal{N}(0, \Omega) \) represents an i.i.d. centered Gaussian noise with known covariance $\Omega$. We further assume that $\Omega = \sigma_w^2I_d$.
The quadratic cost associated to this control is $c(x_t , u_t) = x_t^\intercal Q x_t + u_t^\intercal R u_t,$
where \( Q  \in \mathbb R^{d \times d} \) and \( R  \in \mathbb R^{k \times k} \) are positive definite matrices.
For the rest of the paper, we summarize the system's unknown parameters in $\Theta = (A, B)^\intercal$.  
The infinite horizon average cost function for
a policy $\pi$ specifying the control $u$  in each state $x$
is 
\begin{equation}\label{eq:cost}
    J_\pi(\Theta) = 
    \lim_{T \rightarrow \infty} \frac{1}{T}
    {\mathbb E} \left [ \sum_{t=0}^{T-1} c(x_t, u_t) \right].
\end{equation}
Further, a policy $\pi$ is classically parameterized by a \textit{gain} matrix $K \in \mathbb R^{k \times d}$ as  \( \pi(x_t) = -K x_t \), making it a linear function of the state, with associated cost \eqref{eq:cost}
denoted $J_K(\Theta)$.
Optimal planning can be achieved by solving the Discrete Algebraic Ricatti Equation (DARE), $P = A^\intercal P A + Q - A^\intercal P B \left (B^\intercal P B + R \right)^{-1} B^\intercal P A.
$
We denote the solution of the DARE, $P^\star(\Theta)$, and
the optimal gain that minimizes Eq.
\eqref{eq:cost} is given as 
$
    K^\star(\Theta) = - \left( B^\intercal P^\star(\Theta)B + R\right)^{-1}B^\intercal P^\star(\Theta) A,
$ which achieves the minimal cost $J^\star(\Theta)=J_{K^\star(\Theta)}(\Theta)$. 
When $\Theta$ is clear from context, we simply write $P^\star$, $K^\star, J^\star$.

\paragraph{The learning problem.} We follow the model-based RL setting of \cite{abbasi2011regret}, where parameter $\Theta^\star$ is unknown and $Q$ and $R$ are assumed known. We assume that the system is part of the \textit{stabilizable} set $\mathcal{S}_0$, meaning there exists a gain matrix $K$ such that $A - BK$  is stable, that is with all eigenvalues confined to the interval $(-1, 1)$. It is convenient to introduce the constraint set $\mathcal S \subseteq \mathcal{S}_0 = \{\Theta \in \mathbb{R}^{(k+d)\times d}: J^\star(\Theta) \leq D, \operatorname{Tr}(\Theta\Theta^\intercal) \leq S^2 \}$. At each time $t$ the learner chooses a policy $\pi_t$, observes the current state $x_t$, executes a control $u_t = \pi_t(x_t)$ and incurs the associated cost $c_t = x_t^\intercal Q x_t + u_t^\intercal R u_t$; the system then transitions to the next state $x_{t+1}$. The learning performance is measured by the cumulative regret over $T$ steps defined as 
$
    \mathcal{R}(T) = \sum_{t=0}^T(c_t - J^\star(\Theta^\star)).
$
The unknown parameter $\Theta^\star$ can be directly estimated from sequences $\{x_t, u_t, x_{t+1}\}$ using regularized least-squares (RLS). Let $z_t = (x_t, u_t)^\intercal$, for any regularization paramameter $\lambda \in \mathbb R^+$, the design matrix
and the RLS estimate are defined as
\noindent\begin{minipage}{.5\linewidth}
\begin{equation}
  V_t = \lambda I+ \sum^{t-1}_{s=0}z_sz_s^\intercal,
\end{equation}
\end{minipage}%
\begin{minipage}{.5\linewidth}
\begin{equation} \label{eq:rls}
  \widehat{\Theta}_t = V_t \sum^{t-1}_{s=0}z_sx^\intercal_{s+1}.
\end{equation}
\end{minipage}
Using Theorem 1 from \cite{abbasi2011regret}, for any $\theta \in (0, 1)$, for all $0 \leq t \leq T$, the underlying parameter $\Theta^\star$ lives in the ellipsoid $\mathcal{E}_t(\delta)$ with probability at least $1 - \delta$ where 
$
\mathcal{E}_t(\delta) = \left \{ \Theta^\star \in \mathcal S : \Vert \Theta^\star - \widehat{\Theta}_t \Vert_{V_t} \leq \beta_t(\delta) \right\}, \text{ with }
   \beta_t(\delta) = n \sigma_w \sqrt{2 \log \left( \dfrac{\operatorname{det}(V_t)^{1/2}}{\operatorname{det}(\lambda I)^{1/2}\delta} \right)} + \lambda^{1/2}S.
$
\paragraph{Policy evaluation.} From the form of the policies, it is convenient to introduce $A_K = A - B K$, known as the closed-loop system of $K$. Indeed using this notation, transitions under policy $K$ rewrite $x_{t+1} = A_Kx_t + w_t$, and the discrete-time Bellman equation writes
$
P_K(\Theta) = Q_K + A_K^\intercal P_K(\Theta) A_K,
$
where $Q_K  = Q + K^\intercal R K$ and $P_K(\Theta)$ is the solution to a discrete-time Lyapunov equation. 
We denote the \textit{spectral radius} of a matrix $M$ as $\rho(M)$. If $K$ stabilizes the system, then $\rho(A_K)  < 1$, the cost of $K$ is finite, and $x_t \rightarrow 0$ at a geometric rate.
Under the objective \eqref{eq:cost}, for a gain $K$, a better gain $K'$ ensures
$ J_{K'}(\Theta) \leq J_{K}(\Theta),
$
with $J_{K}(\Theta) = \sigma_w^2 \operatorname{Tr}(P_K(\Theta))$, the average cost of $K$ in $\Theta$.

\paragraph{Optimal MABs strategies.} 
The Minimum Empirical Divergence (\texttt{MED}) algorithm, introduced by \citet{honda2010asymptotically}, achieves asymptotic optimality for MABs. \texttt{MED} derives directly from the fundamental regret lower bound established by \citet{burnetas1996optimal}, which states that for any suboptimal arm \(a \in \mathcal{A}\) (where \(\mu_a < \mu_\star\), with \(\mu_\star\) being the optimal mean), the expected number of pulls \(N_a(T)\) must satisfy:
$
\underset{T\to\infty}{\liminf} \ \mathbb{E}[N_a(T)] \big/{\log T} \geq 1\big/\mathcal{K}_a(b_a,\mu_\star).
$
Here, \(b_a \in \mathcal{D}_a\) represents the reward distribution of arm \(a\), and \(\mathcal{K}_a(b_a,\mu_\star)\) captures the minimum information cost needed to confuse the algorithm between arm \(a\) and a better arm. This is formalized as
$
\mathcal{K}_a(b_a,\mu_\star) = \inf \left\{ \operatorname{KL}(b_a \Vert b) : b \in \mathcal{D}_a,\ \mathbb{E}_{X\sim b}[X] > \mu_\star \right\},
$ where $\operatorname{KL}$ denotes the Kullback-Leibler divergence.
At each time step $t$, \texttt{MED} elegantly transforms this information-theoretic principle into an exploration strategy by sampling arm $a$ with probability proportional to $\exp(-N_a(t)\mathcal{K}_a(\hat{b}_a, \hat{\mu}_\star))$, where notation with 
$\ \hat{} \ $ denotes empirical estimates. The cornerstone of the \texttt{MED} framework is identifying the \textit{confusing instance}, the alternative model that minimizes the $\operatorname{KL}$ divergence while appearing more rewarding than the currently best arm. In the following section, we extend this powerful concept to the substantially more complex setting of LQR.

\section{Efficient Confusing Instance Search for LQR} \label{section-3}

In this section, we now discuss the main insight of our contribution,
borrowing the notion of confusing instances originating from MAB theory to the LQR framework.

\paragraph{Intuition.}
The central element revealing the structure of a sequential decision problem appears when deriving lower bounds on the regret performance of any consistent learner, namely a learner able to achieve optimality on a class of decision problems $\bbM$ rather than a single instance $\bfM\in\bbM$.
The high-level idea is easy to get,
and consists of considering, for a given 
$\bfM\in\bbM$ a policy $\pi$ that isn't optimal in $\bfM$, hence does not achieve minimal cost $J_\star(\bfM)$, where here $\star$ is optimal in  $\bfM$.
We then want to build another MDP $\tbbM$ in which $\pi$ achieves better gain, that is $J_\pi(\tbfM) \leq J_\star(\bfM)$.
Given the multitude of possible MDPs satisfying these conditions, we naturally seek those informationally closest to our initial estimate $\bfM$.
More precisely, the rationale is that if  $\bfM$ and $\tbfM$ are hard to distinguish from playing optimally in $\bfM$, say, from a hypothesis-testing perspective, then any learner that must be optimal in both environments should deviate from playing $\star$.

Formally, let $\Pi^\star(\bfM)= \{ \pi \in\Pi: J_\pi(\bfM)\leq J_{\pi'}(\bfM) \forall \pi'\in\Pi \}$
denote optimal policies for $\bfM$, and alternative models as
$\text{Alt}(\bfM) = \{ \tbfM \in \bbM : \Pi^\star(\bfM)\cap \Pi^\star(\tbfM) = \emptyset\}$.
Introducing $d(\bfM,\tbfM)$ to be e.g. the expected log-likelihood ratio of a trajectory generated from $\Pi^\star(\bfM)$ in both models, we then look for $\tbfM\in \text{Alt}(\bfM)$  minimizing $d(\bfM,\tbfM)$. Such an instance is called confusing or model $\bfM$.

Specializing this approach to LQ systems
introduces both simplifications and challenges.
Interestingly, given $\bfM(\Theta)$, $\Pi^\star(\bfM)$
reduces to $\pi_{K^\star}$, hence we can consider the expected log-likelihood ratio along the trajectory from 
$K^\star$ in both systems.
Note that $K^\star$ must stabilize both systems.

\begin{proposition}[Asymptotic per-step expected log-likelihood ratio for LQR] \label{prop:elr}
Given a gain $K$ that is stabilizing for the two systems $\Theta$ and $\widetilde\Theta$, and assuming both systems share the same covariance matrix $\Omega$, the asymptotic per-step expected log-likelihood  under the two systems is 
\begin{equation}
    {\bf d}_K(\Theta \Vert \widetilde{\Theta}) 
    \stackrel{\text{def}}{=} \lim_{T \rightarrow \infty} \frac{1}{T} \mathbb{E}_\Theta \left [ \log \frac{\mathbf{p}(\tau_T)}{\tilde{\mathbf{p}}(\tau_T)} \right ]= \frac{1}{2} \operatorname{Tr} \left( (A_K - \widetilde{A}_K)^\intercal \Omega^{-1} (A_K - \widetilde{A}_K) \Sigma_K(\Theta) \right ).
\end{equation}
where $\tau_T$ denotes a trajectory of length $T$ from $\pi_K$ and the stationary distribution $\Sigma_K(\Theta)$ induced by $K$ satisfies a discrete-time Lyapunov equation $\Sigma_K(\Theta) = \mathbb E_\Theta \left [x^\intercal_tx_t \right| K] = \Omega + A_K \Sigma_K(\Theta) A_K^\intercal$. 
\end{proposition}
The proof of this proposition is given in Appendix \ref{proof-1}. We now have the necessary elements to tackle the challenge of identifying the most confusing instances in LQR.

\subsection{The Challenge of Approaching the Most Confusing Instance}

Finding the most confusing instance and its associated sub-optimality cost is generally NP-hard. This section introduces key simplifications that yield a computationally efficient approximation.

At a high level, rather than optimizing ${\bf d}_K(\Theta \Vert \widetilde{\Theta})$ over all possible confusing $\tilde \Theta$, we will proceed in Section~\ref{sec:algo} by sampling  a finite set of perturbations
$\Theta'_1,\dots,\Theta'_n$ around a base configuration $\Theta$ and then optimize within the convex hull of these anchor points.
To justify our approach, we analyze an optimization concerning a single perturbation parameter $\Theta'$ of the system. Thanks to the explicit form of optimal policies in LQR, the optimization problem can be formulated as
\begin{equation} \label{eq:lqr-mci-opt}
\underline{\mathbf{K}}(\Theta \Vert \Theta') = \inf_{\widetilde{\Theta}} \{\mathbf d_K(\Theta, \widetilde{\Theta}) \quad
    \text{subject to} \quad J_{K'}(\widetilde\Theta) < J_{K}(\widetilde\Theta)\},
\end{equation}
where we consider two \textit{close} stabilizable instances $\Theta$ and $\Theta'$, with their respective optimal gains $K = K^\star(\Theta)$ and $K' = K^\star(\Theta')$. This objective function is strictly convex in $\widetilde{\Theta}$ for fixed $\Theta$. However, the constraint is non-convex, as the set of stable matrices is generally non-convex. To search for solutions that are both stable with controlled cost, we first observe 
that as each stabilizing LQ system is guaranteed to have a \textit{}{unique} optimal gain that minimizes the associated cost, the cost of $K'$ cannot exceed that of $K$ in $\Theta'$, ensuring that $J_{K'}(\Theta') \leq J_{K}(\Theta')$. (In particular, $\Theta'$ is a feasible solution of the optimization problem defined in Equation $\eqref{eq:lqr-mci-opt}$ and, we get the crude upper bound $\underline{\mathbf{K}}(\Theta \Vert \Theta') \leq \mathbf{d}_K(\Theta, \Theta')$.)

From the preceding initial remark, we thus observed that $J_{K'}(\Theta') - J_{K}(\Theta') < 0$, while $J_{K'}(\Theta) - J_{K}(\Theta) > 0$. This justifies performing a line search interpolating between $\Theta$ and $\Theta'$, effectively reducing the optimization to a one-dimensional search problem, and yielding a reduced upper bound on the sub-optimality cost. 
More formally, we introduce the analytic curve connecting these instances, parametrized by $\alpha \in [0, 1]$ and expressed as $\Theta(\alpha) = (A + \alpha \Delta_A, B + \alpha\Delta_B)$ with $\Delta_A = A' - A$ and $\Delta_B = B' - B$. We then form the following key result.

\begin{proposition}[Sub-optimality cost refinement]\label{prop:sub} Using the linear interpolation parametrization, a valid upper-bound on $\underline{\mathbf{K}}(\Theta \Vert \Theta')$ can be obtained by finding the root of 
    \begin{equation} \label{eq:refinement}
        \mathcal{L} (\alpha) = J_{K'}(\Theta(\alpha)) - J_{K}(\Theta(\alpha)) = 0.
    \end{equation}
\end{proposition}
\begin{proof}
    Assuming $\Theta$ and $\Theta'$ yield different dynamics, $\mathcal{L} (\alpha)$ is a continuous function in $[0, 1]$.
    Now, by definition $\mathcal{L}(0) \cdot \mathcal{L}(1) = \left ( J_{K'}(\Theta) - J_{K}(\Theta) \right ) \cdot \left ( J_{K'}(\Theta') - J_{K}(\Theta') \right )< 0$, because $J_{K'}(\Theta) - J_{K}(\Theta) > 0$ and $J_{K'}(\Theta') - J_{K}(\Theta') < 0$.  This implies that a root exists according to Bolzano's theorem. We can show that $L(\alpha)$ is not convex, but has no local optima, which allows global convergence, as demonstrated in Section 3 of \cite{fazel2018global}. 
    The objective function increases monotonically as $\widetilde{\Theta}$ diverges from $\Theta$ since its derivative equals $t$ times the trace of positive definite matrices' product, ensuring positivity for all $t > 0$.
    Thus, finding the unique root that satisfies the cost constraint is the solution of Eq. \eqref{eq:lqr-mci-opt}, on the linear curve.
\end{proof}
\begin{figure}[h]
    \vspace*{-6mm}
    \centering
    \begin{minipage}{0.325\linewidth}
    \includegraphics[width=\linewidth]{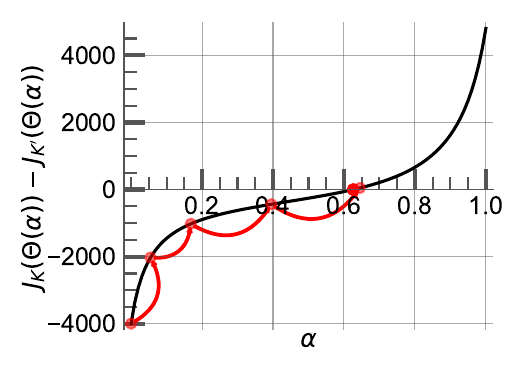}
    \end{minipage}%
    \begin{minipage}{0.675\linewidth}
        \caption{Optimization landscape of the objective $\mathcal{L}(\alpha)$ on two Inverted Pendulum system \citep{barto1983neuronlike} parametrized by the mass and the length of the pendulum. $\Theta$ is parametrized by (.1, .4) and $\Theta'$ by (.3, 1.). Red arrows indicate the Newton steps taken during the optimization process.}
        \label{fig:opt-landscape}
    \end{minipage}
    \vspace*{-6mm}
\end{figure}
This objective is non-convex due to the potential non-stability of the interpolated closed-loop system \citep{lin2009stability},  at the boundary between stable and unstable policies, the objective function quickly becomes infinity. However, $\mathcal{L}(\alpha)$ is \emph{almost} smooth (see Lemma 6 in \cite{fazel2018global}) when the closed-loop systems are not close to the boundary of stability, that allows in practice, to deploy the Newton method, that can solve the objective in few steps, as shown in the Figure \ref{fig:opt-landscape}.

\subsection{Fast Approximate Solution for Small Perturbed Systems}
To find the root of \eqref{eq:refinement}, we introduce a Taylor approximation\footnote{Taylor approximation for finding confusing instance has already been explored by \cite{baudry2023fast} in MABs.} of the objective function. For sufficiently small perturbations $\Theta'$ around $\Theta$, the interpolation between closed-loop systems $A(\alpha) - B(\alpha)K$ and $A(\alpha) - B(\alpha)K'$ remains stable, thanks to the existence of a stability radius \citep{hinrichsen1986stability}. This stability property, supported by perturbation theory, allows us to employ a first-order Taylor expansion to derive a closed-form approximation of Eq. \eqref{eq:refinement}.
\begin{proposition}[Sub-optimality cost refinement under small perturbations] Assume the closed-loop system undergoes perturbations $\Delta_A$ and $\Delta_B$ that are sufficiently small (e.g., $\|\Delta_A\|, \|\Delta_B\| \le \epsilon$ for a small $\epsilon>0$) so that higher-order terms can be neglected, we denote $\Delta_K = \Delta_A + \Delta_BK$ the perturbation on the closed loop system, and the objective $\mathcal{L}(\alpha)$ can be approximated with
\begin{equation}
    \mathcal{L}(\alpha) \approx (p_K - p_{K'}) - \alpha (\overline{p}_K - \overline{p}_{K'})+ \alpha^2(\overline{\overline{p}}_K - \overline{\overline{p}}_{K'}),
\end{equation}
with $p_K = \operatorname{Tr}(P_K(\Theta)),\ \overline{p}_K = \operatorname{Tr}(\overline{P}_K(\Theta)), \ \overline{\overline{p}}_K = \operatorname{Tr}(\overline{\overline{P}}_K(\Theta))$, and $\overline{P}_K(\Theta) = A_K^\intercal \overline{P}_K(\Theta) A_K + A_K^\intercal P_K(\Theta) \Delta_K + \Delta_K^\intercal P_K(\Theta) A_K,\ 
    \overline{\overline{P}}_K(\Theta) =  A_K^\intercal \overline{\overline{P}}_K(\Theta) A_K + \Delta_K^\intercal P_K(\Theta) \Delta_K$.
This is a second-degree polynomial whose coefficients correspond to the trace of the solution of discrete-time Lyapunov equations. The solution can be obtained by identifying the positive root.
\end{proposition}
The proof of this proposition is in \ref{proof-2}. We now have all the elements to design an efficient algorithm.

\section{Towards Minimum Empirical Divergence Strategies for Online LQR} \label{section-4}
\label{sec:algo}

In this section, we introduce \texttt{MED-LQ}, our novel algorithm that extends the asymptotically optimal \texttt{MED} strategy of \citet{honda2011asymptotically} from MABs to LQ systems. Our approach incorporates several adaptations specifically crafted to address the unique challenges of continuous dynamics.

\subsection{The \texttt{MED-LQ} Algorithm} 

\newcommand{\lIfElse}[3]{ #2 \textbf{if}{ (#1)}~{\textbf{else}~#3}}

\begin{algorithm}[H]
\caption{\texttt{MED-LQ}: Minimum Empirical Divergence for Linear Quadratic Systems \label{algo}}
\KwIn{$Q$, $R$, $\widehat{\Theta}_0$, $V_0=\lambda I$, $\delta>0$, $T$, $n$, $\sigma_\eta$, $\sigma_v$, $\epsilon$.}
\For{$t=0,\ldots,T$}{
  \uIf{$\det(V_t)>2\det(V_0)$}{
    Compute $\widehat{\Theta}_t$ via RLS \eqref{eq:rls} and set $\widehat{K}_t=K(\widehat{\Theta}_t)$\;
    Generate $n$ perturbations 
    $
      \left \{W_i=\eta_i\,e_j e_k^\top \mid j,k\sim\mathcal{U}(\{1,\ldots,n\}),\, \eta_i\sim\mathcal{U}(-\sigma_\eta,\sigma_\eta) \right\};
    $ \\
    Form the candidate sets $\{\overline\Theta_i=\widehat{\Theta}_t+W_i\}$  and $ \{K_i=K(\Theta_i)\}$\;
    Define the mask $m_i= m(\overline \Theta_i, \widehat\Theta_t; \epsilon)\in\{0,1\}$ \eqref{eq:mask}\;
    For each candidate with $m_i=1$, compute the $h_i=\mathbf{H}(\widehat{\Theta}_t\parallel \overline \Theta_i;V_t)$ \eqref{eq:med-lq}\;
    Set $\widetilde{\Theta}_t=\widehat{\Theta}_t+\sum_{i=1}^n\omega_i\,W_i$ with $\omega_i=
    \exp(h_i)m_i \Big/\sum_{j=1}^n\exp(h_j)m_j$ and $V_0 = V_t$\;
  }
  \Else{
    Set $\widetilde{\Theta}_t=\widetilde{\Theta}_{t-1}$\;
  }
  Compute the optimal empirical gain $\widetilde{K}_t=K(\widetilde{\Theta}_t)$\;
  Apply $u_t =$ \lIfElse{$\widetilde K_t$ stabilize $\widetilde\Theta_t$}{$\widetilde{K}_t x_t$}{$\widetilde{K}_t x_t + \nu_t$, with $\nu_t \sim \mathcal{N}(0, \sigma^2_\nu)$}\;
  Obtain $x_{t+1}$ and record $(z_t,x_{t+1})$ and update $V_{t+1}=V_t+z_tz_t^\top$;
}
\end{algorithm}

\texttt{MED-LQ} is an online learning algorithm that carefully balances exploration and exploitation in linear dynamical systems. Inspired by the standard learning framework of \cite{abbasi2011regret}, the algorithm proceeds in rounds over a finite horizon. 
At each time step $t$, it first checks whether the accumulated information, quantified by the determinant of the design matrix $\det(V_t)$ has doubled (line 2). When it does, a new optimal empirical parameter $\widehat{\Theta}_t$ is computed using RLS, and the corresponding control gain is derived $\widehat{K}_t=K(\widehat{\Theta}_t)$. To enhance exploration, \texttt{MED-LQ} generates a collection of $n$ candidate parameters $\forall i \in \{0, \cdots, n\}, \ \overline{\Theta}_i = \widehat{\Theta}_t + W_i$ by applying random rank-one perturbations $W_i$ to the RLS estimate (line 4,5). 
Rank-one perturbations simplify the stability analysis, making it tractable \citep{laffey2002stability}, and are inspired by the local-policy search from \citep{pesquerel2021stochastic}.
Each candidate is then filtered through a set of constraints (line 6), to ensure that the most confusing instance search \eqref{eq:lqr-mci-opt} is well-defined.  
The search for the most confusing instance is well-defined when the following  constraints defined by $m(\overline{\Theta}, \widehat{\Theta}; \epsilon)$ hold 
\begin{equation} \label{eq:mask}
\mathbb{I}\bigg\{ \underbrace{\rho(\widehat{A}_{\widehat{K}}) < 1 \land
    \rho(\overline{A}_{\overline{K}}) < 1}_{\text{Closed-loop stability.}} 
    \land \color{gray} \underbrace{\widehat{A}_{\widehat{K}}\widehat{A}_{\overline{K}} \succeq 0 \land \overline{A}_{\widehat{K}}\overline{A}_{\overline{K}} \succeq 0}_{\text{Linear interpolation stability.}} \color{black}
    \land \underbrace{J_{\widehat{K}}(\widehat{\Theta}) - J_{\overline{K}}(\widehat{\Theta}) > \epsilon}_{\text{Alternative set membership.}} \bigg\},
\end{equation}  
where $X \succeq 0$ denotes positive semi-definiteness, and $\epsilon$ is a small threshold value. The first two conditions ensure closed-loop stability. The next two follow from Theorem 1 of \cite{laffey2002stability}, and check that the linear curve between the two closed-loop systems is stable. The last condition checks if $\overline{\Theta}$ belongs to the alternative set of $\widehat\Theta$. 
The linear interpolation stability condition, enabled by our rank-one perturbations, represents a conservative approach. While ensuring stability across the entire interpolation interval exceeds technical requirements, removing this constraint would necessitate computing confusing costs for more instances and implementing careful post-filtering mechanisms. 
We recommend this filtering criterion for computational efficiency, especially in systems with small to moderate dimensions.
For those candidates that pass the stability check, the algorithm evaluates their \emph{Minimum Empirical Divergence} (line 8), which captures the cost of making a perturbed system optimal. This quantity is inspired by  \texttt{MED} and \texttt{LinMED} strategies.
\begin{definition}[Minimum Empirical Divergence coefficients for LQR] During the learning process, where $V_t$ represent the design matrix at time $t$, $\widehat \Theta_t$ the empirical optimal RLS estimate and $\Theta$ an alternative parameter, the minimum empirical divergence is given by
    \begin{equation} \label{eq:med-lq}
        \mathbf{H}_t(\Theta) = - \frac{\mathbf{K}(\widehat\Theta_t \Vert \Theta)}{\Vert \Theta \Vert^2_{V_{t}^{-1}}}.
    \end{equation}
\end{definition}
\texttt{MED-LQ} generates exponential weights (line 8) to create a weighted combination of perturbations, biasing parameter estimates toward candidates with lower divergence values.
Finally, the corresponding control gain is applied to the system. 
We introduce additional isotropic exploration noise $\nu_t$, similarly to \cite{tu2019gap,lale2022reinforcement,pmlr-v178-kargin22a}, when the empirical gain fails to stabilize the empirical estimate, which intuitively happens mainly in the early rounds.
This noise provides excitation, ensuring the identifiability of the system dynamics by exploring the state-space in all directions.
Finally, new state data is collected to update the design matrix, thus refining the parameter estimates over time. The full algorithm is summarized in Algorithm \ref{algo}.

\subsection{Intuition and design elements} 
Let us now provide insights and sketch the main ideas supporting the soundness of this strategy.
\texttt{MED-LQ} extends the asymptotically optimal \texttt{IMED-RL} algorithm \citep{pesquerel2022imed} for ergodic discrete MDPs to the LQR setting while incorporating continuous aspects developed in \texttt{LinMED} \citep{balagopalan2024minimum} for linear sub-Gaussian MABs. Both methods leverage regret lower bounds to achieve superior efficiency compared to OFU-based approaches, with \texttt{IMED} being the deterministic counterpart of \texttt{MED}.
\paragraph{Ergodicity and information gain.}
In \texttt{IMED-RL}, ergodicity ensures that every policy eventually visits all states, enabling efficient information gathering across the state space. For linear dynamical systems, the situation is comparable: observing a single state can provide global insights about system dynamics, similar to the information transfer in linear bandits. However, since quantifying the per-step information gain is challenging, we execute each chosen policy for multiple steps until a significant change in information volume occurs (line 2).
\paragraph{Policy improvement.}
A cornerstone of \texttt{IMED-RL} is exploiting the policy improvement property from \cite{puterman2014markov}, which guarantees that in discrete ergodic MDPs, any sub-optimal policy can be improved through a \emph{local} (single-state) modification, a convenient property not universally applicable. 
This approach efficiently identifies confusing instances by searching only over local policy modifications, with central analysis demonstrating a high probability of policy improvement.
For linear-quadratic systems, we identify single entry-wise perturbations of the system matrix as the natural equivalent to single-state modifications. This approach yields substantial computational benefits, as candidate perturbations become straightforward to generate. However, rather than directly applying single-entry perturbations, which alone may be insufficient to guarantee policy improvement, we form convex combinations of candidates weighted by MED coefficients. This strategic convex combination substantially expands the search space volume, significantly increasing the probability of discovering effective policy improvements.
\paragraph{Policy gradient.}
While \texttt{IMED-RL} estimates an empirical MDP, applies value iteration, and selects actions minimizing the IMED index, \texttt{MED-LQ} follows a parallel approach. We estimate system parameters via RLS, solve the DARE to capture the value function, and define our minimum empirical divergence analogously to \texttt{IMED-RL}. Inspired by \texttt{LinMED}, the term $1 \big/ \Vert \Theta\Vert ^2_{V_t^{-1}}$ effectively functions as a visitation count analog. Conceptually, where \texttt{IMED-RL} implements policy iteration, \texttt{MED-LQ} adopts an approximate policy gradient approach. The fundamental intuition is that a policy's selection likelihood should at least match its posterior probability of optimality, with perturbations directing exploration toward promising regions of the policy space.
\paragraph{Continuum policy and $\epsilon$-optimality.} The policy improvement lemma from \cite{puterman2014markov} applies to discrete, ergodic MDPs, where finitely many policies ensure that a finite number of improvement steps reach optimality. This property doesn't extend to continuous MDPs with infinite policy sets. By introducing the parameter $\epsilon$ in our filtering condition, we effectively consider $\epsilon$-near-optimal policies rather than strictly optimal ones, implicitly covering the policy space with finitely many level sets. This approach ensures that finitely many $\epsilon$-policy-improvement steps yield a near-optimal policy.
In practice, $\epsilon$ requires careful calibration: not too small (to ensure non-empty filtered sets) and not too large (to avoid requiring excessive policy-improvement steps). We recommend $\epsilon$ that scales between $O(1/T)$ and $O(1/\log^2(T))$.
\paragraph{Excitation.}
A well-known challenge in LQR is the initial information scarcity that impedes the invertibility of matrices defining stable policies. This challenge dissipates after sufficient observations span the entire state space, after adequate system excitation. In line (13), we introduce noise $\nu_t$ to enforce excitation whenever the control fails to stabilize the confusing instance. This mechanism primarily induces additional exploration during early rounds, while in the asymptotic regime, all selected policies naturally stabilize the system, eliminating the need for artificial excitation. In our Section \ref{section-5}, we study the effect of excitation under the name of "auto-stabilization".

Following our insights, a rigorous regret analysis of \texttt{MED-LQ} presents unique theoretical challenges distinct from \texttt{MED}, \texttt{IMED-RL}, or \texttt{LinMED}. Two critical questions emerge: (1) establishing that \texttt{MED-LQ} guarantees high-probability policy improvements with sufficient margin at each iteration, and (2) determining the precise magnitude of entry-wise perturbations needed to ensure policy improvements exist within local neighborhoods. These challenges require adapting policy-improvement arguments to continuous settings, a non-trivial extension demanding specialized analysis beyond this paper's scope. While \texttt{MED-LQ} deliberately addresses these challenges through techniques such as combining multiple single-entry perturbations, we reserve a comprehensive theoretical analysis for future work.

\section{Numerical Experiments} \label{section-5}

\begin{figure}[htbp]
    \centering
    \begin{subfigure}{.98\linewidth}
        \centering
        \includegraphics[width=1.\linewidth]{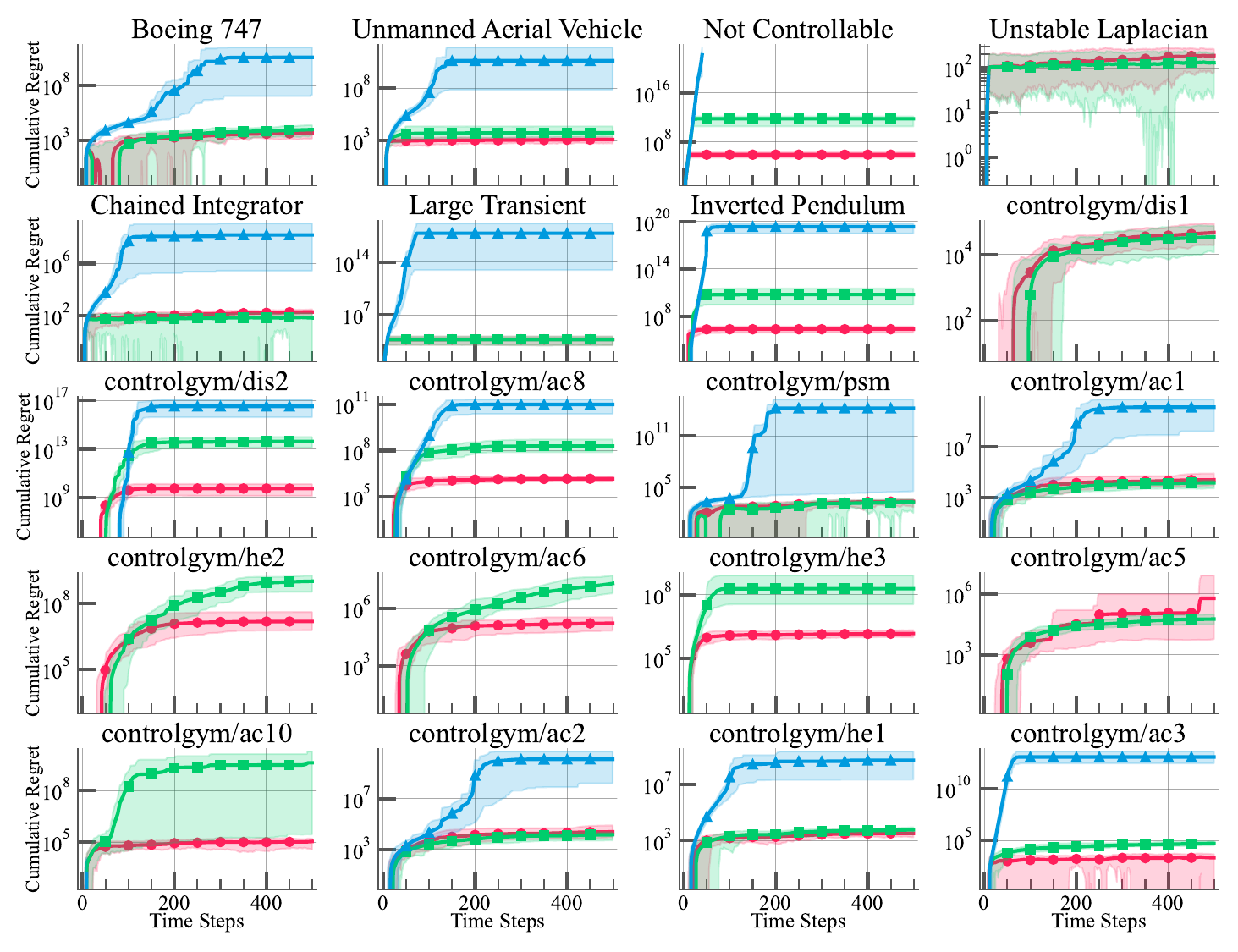}
        \includegraphics[scale=0.75]{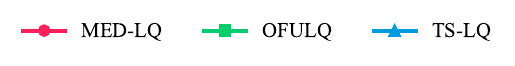}
        \caption{Comparison of \texttt{MED-LQ}, \texttt{OFULQ}, and \texttt{TS-LQ} initialized with a stable controller.}
        \label{fig:scenario1}
    \end{subfigure}
    \begin{subfigure}{.98\linewidth}
        \centering
        \includegraphics[width=1.\linewidth]{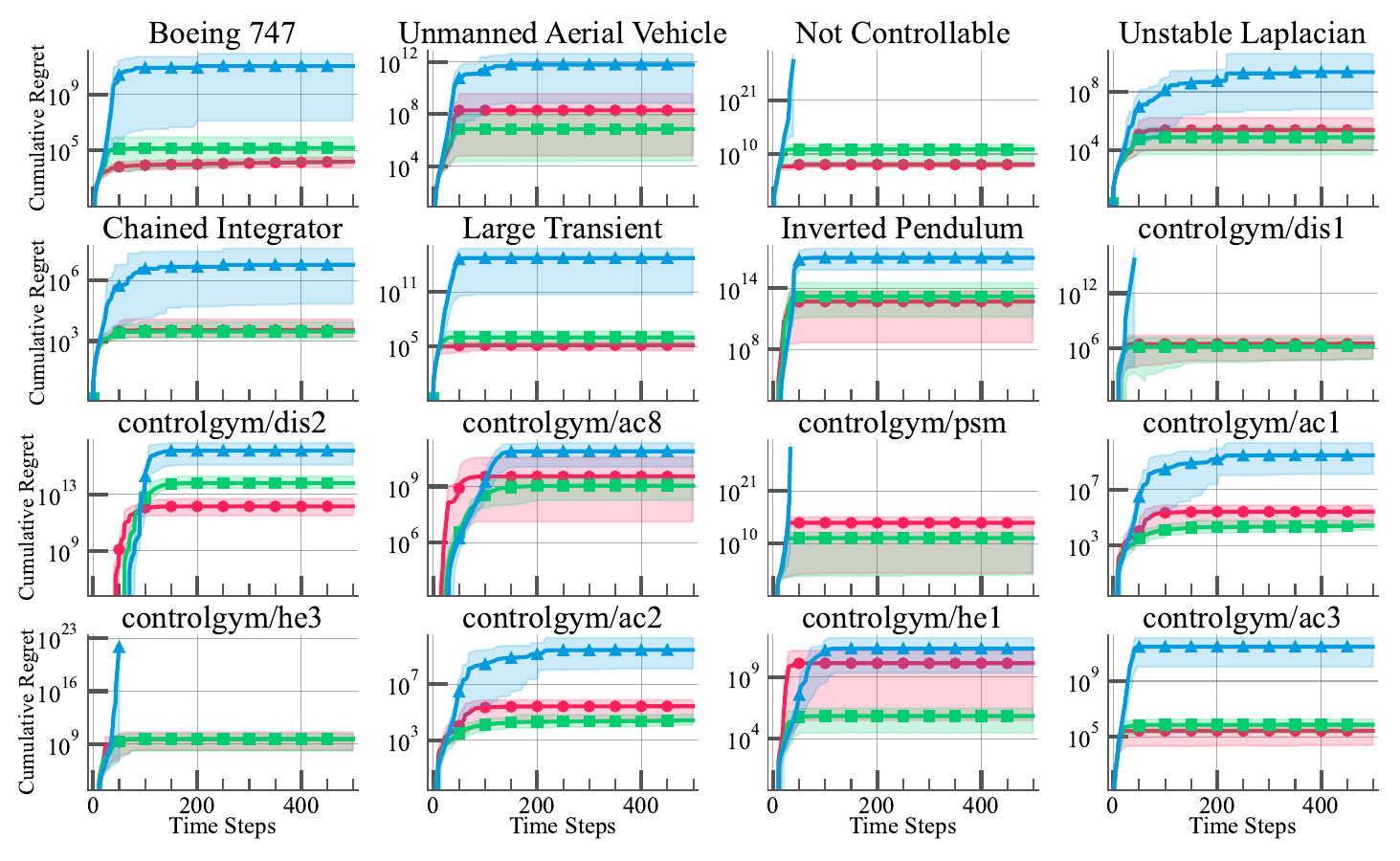}
        \includegraphics[scale=0.75]{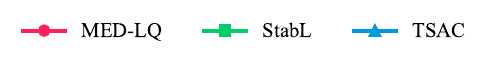}
        \caption{Comparison of \texttt{MED-LQ}, \texttt{StabL}, and \texttt{TSAC} in the auto-stabilization scenario}
        \label{fig:scenario2}
    \end{subfigure}
    \caption{Performance comparison of \texttt{MED-LQ} under two distinct initialization scenarios.}
    \label{fig:enter-label}
\end{figure}

To study the numerical potential of \texttt{MED-LQ}, we evaluate it on a control suite that includes classic environments from the online LQR literature, such as the Boeing 747 and Unmanned Aerial Vehicle, as well as additional industrial control problems from \texttt{controlgym} \citep{zhang2023controlgym}, inspired by real-world applications. All environments are subject to a normal noise $\mathcal{N}(0, 1)$ and have a moderate size (from 2 to 10 dims). We assess the performance of \texttt{MED-LQ} in two distinct scenarios.
\begin{description}
    \item[Scenario 1: Stable Initialization.] In this setting, we initialize the algorithm with a stable controller and seed the dataset with a trajectory of 50 time steps. We compare \texttt{MED-LQ} against \texttt{OFULQ} \citep{abbasi2011regret} and \texttt{TS-LQ} \citep{abeille2017thompson}. The stable initialization allows us to assess the exploration efficiency and convergence properties when the system starts in a well-controlled regime. The results are shown in Figure \ref{fig:scenario1}.
    \item[Scenario 2: Auto-Stabilization] Here, \texttt{MED-LQ} is deployed with an initial parameter estimate $\widehat{\Theta}_0 = \mathbf{0}$. To facilitate auto-stabilization, the policy is executed with isotropic noise $w \sim \mathcal{N}(0, 1)$ for the first 35 time steps, as in \cite{lale2022reinforcement}. We compare \texttt{MED-LQ} against \texttt{StabL} \citep{lale2022reinforcement} and \texttt{TSAC} \citep{pmlr-v178-kargin22a}, the auto-stabilizing counterparts of \texttt{OFULQ} and \texttt{TS-LQ}, respectively. The results are shown in Figure \ref{fig:scenario2}.
\end{description}

\paragraph{Implementation details.} 
We implement all baselines within the JAX framework \citep{jax2018github} using a new library, \texttt{linquax}\footnote{A WIP version of the library is available in \url{https://anonymous.4open.science/r/linquax-4FCF/}.}, which delivers highly performant online LQR algorithms with GPU/TPU support and automatic differentiation. In our implementation, \texttt{OFULQ} and \texttt{StabL} are optimized via projected gradient descent, while \texttt{TS-LQ} and \texttt{TSAC} employ a rejection sampling operator. In addition to the doubling trick, we enforce a minimum patience period of 10 steps to prevent excessive early updates that can lead to increased regret. All algorithms share common hyperparameters, chosen after previous work, with $\lambda = 1\times10^{-4}$ and $\delta = 1\times10^{-4}$. For \texttt{MED-LQ}, we define without hyperparameter search the number of candidates $n=128$ and $\sigma_\eta=1$.  Experiments were conducted in less than 1 hour, on a CPU-only cluster equipped with four 64-core AMD Zen3 processors. For classic environments, we used 64 random seeds, and for \texttt{ controlgym} environments, 48 seeds. Performance metrics are reported as the interquartile mean along with the 25th percentile and 75th percentile for each experiment.

\paragraph{Discussion of results.} We compare \texttt{MED-LQ} against \texttt{OFULQ}, \texttt{TS-LQ}, \texttt{StabL}, and \texttt{TSAC}. Our experimental evaluation reveals that \texttt{MED-LQ} demonstrates strong performance across environments.
With stable initialization, \texttt{MED-LQ} shows rapid convergence to low cumulative regret, validating that CI-guided exploration effectively balances exploration and exploitation. All algorithms benefit from stable initialization, allowing them to focus on policy refinement rather than basic stabilization.
In zero-knowledge settings requiring auto-stabilization, \texttt{MED-LQ} quickly discovers stabilizing policies. It consistently outperforms Thompson Sampling methods, which occasionally fail to find stabilizing controllers even after 10,000 rejection sampling attempts. Compared to state-of-the-art methods \texttt{OFULQ} and \texttt{StabL}, \texttt{MED-LQ} demonstrates superior efficiency in most environments, matching \texttt{StabL}'s performance in others, with the sole exception being the \texttt{controlgym/he1} environment under auto-stabilization.
These results establish \texttt{MED-LQ} as a competitive and reliable alternative to OFU-based and Thompson Sampling approaches for online LQR tasks.

\paragraph{Sample size study.} We now examine how the sample size used in \texttt{MED-LQ} affects both regret and execution time in the Inverted Pendulum environment. Experiments were run on a NVIDIA A100 GPU. Figure \ref{fig:abla} presents the results. The plot on the left shows that runtime remains relatively constant across different sample sizes (0.3-0.5 seconds), highlighting the parallelization capabilities of our GPU implementation. The right plot shows that increasing the sample size leads to slightly lower regret until approximately 64 samples, after which the performance plateaus. This suggests that in Inverted Pendulum 64 samples are sufficient to adequately span the space of candidate policies.

\begin{figure}[htpb]
    \centering
    \includegraphics[width=0.8\linewidth]{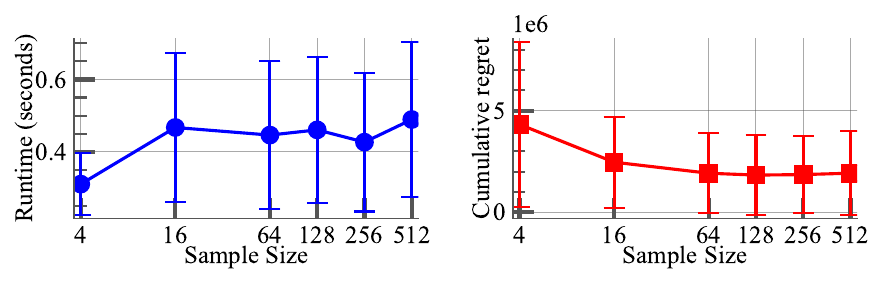}
    \caption{Study of the sample size on the Inverted Pendulum environment.}
    \label{fig:abla}
\end{figure}

\newpage
\section{Conclusion}

In this work, we introduced the Confusing Instance (CI) principle as a novel approach to exploration in online Linear Quadratic Control (LQR). By extending the Minimum Empirical Divergence (MED) framework beyond discrete settings, we developed \texttt{MED-LQ}, the first method to apply the confusing instance principle beyond tabular MDPs.
Our approach employs strategically designed rank-one and entry-wise perturbations that enable efficient identification of confusing instances while maintaining computational feasibility. Notably, \texttt{MED-LQ} avoids confidence bounds (intractable in large spaces) and instead relies on the policy iteration framework. Our methodology is generalizable to other settings: compute empirical optimal policy, generate candidates, approximate confusing instances, compute the minimum empirical divergence, and update policy toward areas minimizing this divergence.
Benchmarks demonstrate that \texttt{MED-LQ} matches state-of-the-art performance, overcoming limitations of existing methods such as OFU and TS. 
\begin{center}
    \emph{We believe that the CI principle deserves greater attention as it introduces a fresh perspective on exploration in continuous MDPs. Our work establishes the foundations for this promising approach, opening new avenues for exploration strategies in complex problems.}
\end{center}

\paragraph{Future Work.}
Future research should refine \texttt{MED-LQ}'s theoretical foundations by establishing formal regret bounds and analyzing the minimal perturbation magnitudes needed for guaranteed policy improvements. A particularly promising direction is to extend the CI principle to high-dimensional problems in deep RL, where efficient exploration remains challenging. The principles established here provide a foundation for novel exploration strategies in both continuous control and complex decision-making tasks.

\subsubsection*{Broader Impact Statement}
\label{sec:broaderImpact}
Our work on efficient exploration in LQR systems has potential applications in robotics, autonomous vehicles, and industrial control systems. While our algorithm enables more efficient learning in these domains, it could also accelerate the deployment of autonomous systems with inherent safety considerations. We advocate for robust safety validation before deploying such learning-based controllers in critical applications.

\section*{Acknowledgments}
This work has been supported by the French Ministry of Higher Education and Research, the Hauts-de-France region, Inria, and the MEL. Additional support was provided by the French National Research Agency under the PEPR IA FOUNDRY project (ANR-23-PEIA-0003).
Experiments presented in this paper were carried out using the PlaFRIM experimental testbed, supported by Inria, CNRS (LABRI and IMB), Université de Bordeaux, Bordeaux INP and Conseil Régional d'Aquitaine (see \url{https://www.plafrim.fr}).
The authors are affiliated with the Inria Scool team project.






\bibliography{main}

\begin{thebibliography}{57}
\providecommand{\natexlab}[1]{#1}
\providecommand{\url}[1]{\texttt{#1}}
\expandafter\ifx\csname urlstyle\endcsname\relax
  \providecommand{\doi}[1]{DOI: #1}\else
  \providecommand{\doi}{DOI: \begingroup \urlstyle{rm}\Url}\fi

\bibitem[Abbasi-Yadkori \& Szepesv{\'a}ri(2011)Abbasi-Yadkori and
  Szepesv{\'a}ri]{abbasi2011regret}
Yasin Abbasi-Yadkori and Csaba Szepesv{\'a}ri.
\newblock Regret bounds for the adaptive control of linear quadratic systems.
\newblock In \emph{Proceedings of the 24th Annual Conference on Learning
  Theory}, pp.\  1--26. JMLR Workshop and Conference Proceedings, 2011.

\bibitem[Abbasi-Yadkori et~al.(2011)Abbasi-Yadkori, P{\'a}l, and
  Szepesv{\'a}ri]{abbasi2011improved}
Yasin Abbasi-Yadkori, D{\'a}vid P{\'a}l, and Csaba Szepesv{\'a}ri.
\newblock Improved algorithms for linear stochastic bandits.
\newblock \emph{Advances in neural information processing systems}, 24, 2011.

\bibitem[Abeille \& Lazaric(2017{\natexlab{a}})Abeille and
  Lazaric]{abeille2017linear}
Marc Abeille and Alessandro Lazaric.
\newblock Linear thompson sampling revisited.
\newblock In \emph{Artificial Intelligence and Statistics}, pp.\  176--184.
  PMLR, 2017{\natexlab{a}}.

\bibitem[Abeille \& Lazaric(2017{\natexlab{b}})Abeille and
  Lazaric]{abeille2017thompson}
Marc Abeille and Alessandro Lazaric.
\newblock Thompson sampling for linear-quadratic control problems.
\newblock In \emph{Artificial intelligence and statistics}, pp.\  1246--1254.
  PMLR, 2017{\natexlab{b}}.

\bibitem[Abeille \& Lazaric(2018)Abeille and Lazaric]{abeille2018improved}
Marc Abeille and Alessandro Lazaric.
\newblock Improved regret bounds for thompson sampling in linear quadratic
  control problems.
\newblock In \emph{International Conference on Machine Learning}, pp.\  1--9.
  PMLR, 2018.

\bibitem[Abeille \& Lazaric(2020)Abeille and Lazaric]{abeille2020efficient}
Marc Abeille and Alessandro Lazaric.
\newblock Efficient optimistic exploration in linear-quadratic regulators via
  lagrangian relaxation.
\newblock In \emph{International Conference on Machine Learning}, pp.\  23--31.
  PMLR, 2020.

\bibitem[Agrawal \& Goyal(2013)Agrawal and Goyal]{agrawal2013thompson}
Shipra Agrawal and Navin Goyal.
\newblock Thompson sampling for contextual bandits with linear payoffs.
\newblock In \emph{International conference on machine learning}, pp.\
  127--135. PMLR, 2013.

\bibitem[Auer \& Ortner(2006)Auer and Ortner]{auer2006logarithmic}
Peter Auer and Ronald Ortner.
\newblock Logarithmic online regret bounds for undiscounted reinforcement
  learning.
\newblock \emph{Advances in neural information processing systems}, 19, 2006.

\bibitem[Auer et~al.(2002)Auer, Cesa-Bianchi, and Fischer]{auer2002finite}
Peter Auer, Nicolo Cesa-Bianchi, and Paul Fischer.
\newblock Finite-time analysis of the multiarmed bandit problem.
\newblock \emph{Machine learning}, 47:\penalty0 235--256, 2002.

\bibitem[Auer et~al.(2008)Auer, Jaksch, and Ortner]{auer2008near}
Peter Auer, Thomas Jaksch, and Ronald Ortner.
\newblock Near-optimal regret bounds for reinforcement learning.
\newblock \emph{Advances in neural information processing systems}, 21, 2008.

\bibitem[Azizzadenesheli et~al.(2018)Azizzadenesheli, Brunskill, and
  Anandkumar]{azizzadenesheli2018efficient}
Kamyar Azizzadenesheli, Emma Brunskill, and Animashree Anandkumar.
\newblock Efficient exploration through bayesian deep q-networks.
\newblock In \emph{2018 Information Theory and Applications Workshop (ITA)},
  pp.\  1--9. IEEE, 2018.

\bibitem[Balagopalan \& Jun(2024)Balagopalan and Jun]{balagopalan2024minimum}
Kapilan Balagopalan and Kwang-Sung Jun.
\newblock Minimum empirical divergence for sub-gaussian linear bandits.
\newblock \emph{arXiv preprint arXiv:2411.00229}, 2024.

\bibitem[Barto et~al.(1983)Barto, Sutton, and Anderson]{barto1983neuronlike}
Andrew~G Barto, Richard~S Sutton, and Charles~W Anderson.
\newblock Neuronlike adaptive elements that can solve difficult learning
  control problems.
\newblock \emph{IEEE transactions on systems, man, and cybernetics}, \penalty0
  (5):\penalty0 834--846, 1983.

\bibitem[Baudry et~al.(2023{\natexlab{a}})Baudry, Pesquerel, Degenne, and
  Maillard]{baudry2023fast}
Dorian Baudry, Fabien Pesquerel, R{\'e}my Degenne, and Odalric-Ambrym Maillard.
\newblock Fast asymptotically optimal algorithms for non-parametric stochastic
  bandits.
\newblock \emph{Advances in Neural Information Processing Systems},
  36:\penalty0 11469--11514, 2023{\natexlab{a}}.

\bibitem[Baudry et~al.(2023{\natexlab{b}})Baudry, Suzuki, and
  Honda]{baudry2023general}
Dorian Baudry, Kazuya Suzuki, and Junya Honda.
\newblock A general recipe for the analysis of randomized multi-armed bandit
  algorithms.
\newblock \emph{arXiv preprint arXiv:2303.06058}, 2023{\natexlab{b}}.

\bibitem[Bellemare et~al.(2016)Bellemare, Srinivasan, Ostrovski, Schaul,
  Saxton, and Munos]{bellemare2016unifying}
Marc Bellemare, Sriram Srinivasan, Georg Ostrovski, Tom Schaul, David Saxton,
  and Remi Munos.
\newblock Unifying count-based exploration and intrinsic motivation.
\newblock \emph{Advances in neural information processing systems}, 29, 2016.

\bibitem[Bertsekas(2012)]{bertsekas2012dynamic}
Dimitri Bertsekas.
\newblock \emph{Dynamic programming and optimal control: Volume I}, volume~4.
\newblock Athena scientific, 2012.

\bibitem[Bian \& Jun(2022)Bian and Jun]{bian2022maillard}
Jie Bian and Kwang-Sung Jun.
\newblock Maillard sampling: Boltzmann exploration done optimally.
\newblock In \emph{International Conference on Artificial Intelligence and
  Statistics}, pp.\  54--72. PMLR, 2022.

\bibitem[Bian \& Tan(2024)Bian and Tan]{bian2024indexed}
Jie Bian and Vincent~YF Tan.
\newblock Indexed minimum empirical divergence-based algorithms for linear
  bandits.
\newblock \emph{arXiv preprint arXiv:2405.15200}, 2024.

\bibitem[Boone \& Maillard(2025)Boone and Maillard]{boone2025regret}
Victor Boone and Odalric-Ambrym Maillard.
\newblock The regret lower bound for communicating markov decision processes.
\newblock \emph{arXiv preprint arXiv:2501.13013}, 2025.

\bibitem[Bourel et~al.(2020)Bourel, Maillard, and Talebi]{bourel2020tightening}
Hippolyte Bourel, Odalric Maillard, and Mohammad~Sadegh Talebi.
\newblock Tightening exploration in upper confidence reinforcement learning.
\newblock In \emph{International Conference on Machine Learning}, pp.\
  1056--1066. PMLR, 2020.

\bibitem[Bradbury et~al.(2018)Bradbury, Frostig, Hawkins, Johnson, Leary,
  Maclaurin, Necula, Paszke, Vander{P}las, Wanderman-{M}ilne, and
  Zhang]{jax2018github}
James Bradbury, Roy Frostig, Peter Hawkins, Matthew~James Johnson, Chris Leary,
  Dougal Maclaurin, George Necula, Adam Paszke, Jake Vander{P}las, Skye
  Wanderman-{M}ilne, and Qiao Zhang.
\newblock {JAX}: composable transformations of {P}ython+{N}um{P}y programs,
  2018.
\newblock URL \url{http://github.com/jax-ml/jax}.

\bibitem[Burda et~al.(2018)Burda, Edwards, Storkey, and
  Klimov]{burda2018exploration}
Yuri Burda, Harrison Edwards, Amos Storkey, and Oleg Klimov.
\newblock Exploration by random network distillation.
\newblock \emph{arXiv preprint arXiv:1810.12894}, 2018.

\bibitem[Burnetas \& Katehakis(1996)Burnetas and
  Katehakis]{burnetas1996optimal}
Apostolos~N Burnetas and Michael~N Katehakis.
\newblock Optimal adaptive policies for sequential allocation problems.
\newblock \emph{Advances in Applied Mathematics}, 17\penalty0 (2):\penalty0
  122--142, 1996.

\bibitem[Burnetas \& Katehakis(1997)Burnetas and
  Katehakis]{burnetas1997optimal}
Apostolos~N Burnetas and Michael~N Katehakis.
\newblock Optimal adaptive policies for markov decision processes.
\newblock \emph{Mathematics of Operations Research}, 22\penalty0 (1):\penalty0
  222--255, 1997.

\bibitem[Cohen et~al.(2018)Cohen, Hasidim, Koren, Lazic, Mansour, and
  Talwar]{cohen2018online}
Alon Cohen, Avinatan Hasidim, Tomer Koren, Nevena Lazic, Yishay Mansour, and
  Kunal Talwar.
\newblock Online linear quadratic control.
\newblock In \emph{International Conference on Machine Learning}, pp.\
  1029--1038. PMLR, 2018.

\bibitem[Curi et~al.(2020)Curi, Berkenkamp, and Krause]{curi2020efficient}
Sebastian Curi, Felix Berkenkamp, and Andreas Krause.
\newblock Efficient model-based reinforcement learning through optimistic
  policy search and planning.
\newblock \emph{Advances in Neural Information Processing Systems},
  33:\penalty0 14156--14170, 2020.

\bibitem[Fazel et~al.(2018)Fazel, Ge, Kakade, and Mesbahi]{fazel2018global}
Maryam Fazel, Rong Ge, Sham Kakade, and Mehran Mesbahi.
\newblock Global convergence of policy gradient methods for the linear
  quadratic regulator.
\newblock In \emph{International conference on machine learning}, pp.\
  1467--1476. PMLR, 2018.

\bibitem[Hinrichsen \& Pritchard(1986)Hinrichsen and
  Pritchard]{hinrichsen1986stability}
Diederich Hinrichsen and Anthony~J Pritchard.
\newblock Stability radius for structured perturbations and the algebraic
  riccati equation.
\newblock \emph{Systems \& Control Letters}, 8\penalty0 (2):\penalty0 105--113,
  1986.

\bibitem[Honda \& Takemura(2010)Honda and Takemura]{honda2010asymptotically}
Junya Honda and Akimichi Takemura.
\newblock An asymptotically optimal bandit algorithm for bounded support
  models.
\newblock In \emph{COLT}, pp.\  67--79. Citeseer, 2010.

\bibitem[Honda \& Takemura(2011)Honda and Takemura]{honda2011asymptotically}
Junya Honda and Akimichi Takemura.
\newblock An asymptotically optimal policy for finite support models in the
  multiarmed bandit problem.
\newblock \emph{Machine Learning}, 85:\penalty0 361--391, 2011.

\bibitem[Honda \& Takemura(2015)Honda and Takemura]{honda2015non}
Junya Honda and Akimichi Takemura.
\newblock Non-asymptotic analysis of a new bandit algorithm for semi-bounded
  rewards.
\newblock \emph{J. Mach. Learn. Res.}, 16:\penalty0 3721--3756, 2015.

\bibitem[Kargin et~al.(2022)Kargin, Lale, Azizzadenesheli, Anandkumar, and
  Hassibi]{pmlr-v178-kargin22a}
Taylan Kargin, Sahin Lale, Kamyar Azizzadenesheli, Animashree Anandkumar, and
  Babak Hassibi.
\newblock Thompson sampling achieves $\tilde{O}(\sqrt{T})$ regret in linear
  quadratic control.
\newblock In Po-Ling Loh and Maxim Raginsky (eds.), \emph{Proceedings of Thirty
  Fifth Conference on Learning Theory}, volume 178 of \emph{Proceedings of
  Machine Learning Research}, pp.\  3235--3284. PMLR, 02--05 Jul 2022.

\bibitem[Kaufmann et~al.(2012)Kaufmann, Korda, and Munos]{kaufmann2012thompson}
Emilie Kaufmann, Nathaniel Korda, and R{\'e}mi Munos.
\newblock Thompson sampling: An asymptotically optimal finite-time analysis.
\newblock In \emph{International conference on algorithmic learning theory},
  pp.\  199--213. Springer, 2012.

\bibitem[Ladosz et~al.(2022)Ladosz, Weng, Kim, and Oh]{ladosz2022exploration}
Pawel Ladosz, Lilian Weng, Minwoo Kim, and Hyondong Oh.
\newblock Exploration in deep reinforcement learning: A survey.
\newblock \emph{Information Fusion}, 85:\penalty0 1--22, 2022.

\bibitem[Laffey et~al.(2002)Laffey, Shorten, and Cairbre]{laffey2002stability}
Thomas Laffey, Robert Shorten, and Fiacre~O Cairbre.
\newblock On the stability of convex sums of rank-1 perturbed matrices.
\newblock In \emph{Proceedings of the 2002 American Control Conference (IEEE
  Cat. No. CH37301)}, volume~2, pp.\  1246--1247. IEEE, 2002.

\bibitem[Lai \& Robbins(1985)Lai and Robbins]{lai1985asymptotically}
Tze~Leung Lai and Herbert Robbins.
\newblock Asymptotically efficient adaptive allocation rules.
\newblock \emph{Advances in applied mathematics}, 6\penalty0 (1):\penalty0
  4--22, 1985.

\bibitem[Lale et~al.(2022)Lale, Azizzadenesheli, Hassibi, and
  Anandkumar]{lale2022reinforcement}
Sahin Lale, Kamyar Azizzadenesheli, Babak Hassibi, and Animashree Anandkumar.
\newblock Reinforcement learning with fast stabilization in linear dynamical
  systems.
\newblock In \emph{International Conference on Artificial Intelligence and
  Statistics}, pp.\  5354--5390. PMLR, 2022.

\bibitem[Lattimore \& Szepesvari(2017)Lattimore and
  Szepesvari]{lattimore2017end}
Tor Lattimore and Csaba Szepesvari.
\newblock The end of optimism? an asymptotic analysis of finite-armed linear
  bandits.
\newblock In \emph{Artificial Intelligence and Statistics}, pp.\  728--737.
  PMLR, 2017.

\bibitem[Lin \& Antsaklis(2009)Lin and Antsaklis]{lin2009stability}
Hai Lin and Panos~J Antsaklis.
\newblock Stability and stabilizability of switched linear systems: a survey of
  recent results.
\newblock \emph{IEEE Transactions on Automatic control}, 54\penalty0
  (2):\penalty0 308--322, 2009.

\bibitem[Maran et~al.(2025)Maran, Metelli, Papini, and
  Restelli]{maran2025local}
Davide Maran, Alberto~Maria Metelli, Matteo Papini, and Marcello Restelli.
\newblock Local linearity: the key for no-regret reinforcement learning in
  continuous mdps.
\newblock \emph{Advances in Neural Information Processing Systems},
  37:\penalty0 75986--76029, 2025.

\bibitem[Mete et~al.(2022)Mete, Singh, and Kumar]{mete2022augmented}
Akshay Mete, Rahul Singh, and PR~Kumar.
\newblock Augmented rbmle-ucb approach for adaptive control of linear quadratic
  systems.
\newblock \emph{Advances in Neural Information Processing Systems},
  35:\penalty0 9302--9314, 2022.

\bibitem[Osband \& Van~Roy(2017)Osband and Van~Roy]{osband2017posterior}
Ian Osband and Benjamin Van~Roy.
\newblock Why is posterior sampling better than optimism for reinforcement
  learning?
\newblock In \emph{International conference on machine learning}, pp.\
  2701--2710. PMLR, 2017.

\bibitem[Osband et~al.(2013)Osband, Russo, and Van~Roy]{osband2013more}
Ian Osband, Daniel Russo, and Benjamin Van~Roy.
\newblock (more) efficient reinforcement learning via posterior sampling.
\newblock \emph{Advances in Neural Information Processing Systems}, 26, 2013.

\bibitem[Osband et~al.(2016)Osband, Blundell, Pritzel, and
  Van~Roy]{osband2016deep}
Ian Osband, Charles Blundell, Alexander Pritzel, and Benjamin Van~Roy.
\newblock Deep exploration via bootstrapped dqn.
\newblock \emph{Advances in neural information processing systems}, 29, 2016.

\bibitem[Pesquerel \& Maillard(2022)Pesquerel and Maillard]{pesquerel2022imed}
Fabien Pesquerel and Odalric-Ambrym Maillard.
\newblock Imed-rl: Regret optimal learning of ergodic markov decision
  processes.
\newblock \emph{Advances in Neural Information Processing Systems},
  35:\penalty0 26363--26374, 2022.

\bibitem[Pesquerel et~al.(2021)Pesquerel, Saber, and
  Maillard]{pesquerel2021stochastic}
Fabien Pesquerel, Hassan Saber, and Odalric-Ambrym Maillard.
\newblock Stochastic bandits with groups of similar arms.
\newblock \emph{Advances in Neural Information Processing Systems},
  34:\penalty0 19461--19472, 2021.

\bibitem[Puterman(2014)]{puterman2014markov}
Martin~L Puterman.
\newblock \emph{Markov decision processes: discrete stochastic dynamic
  programming}.
\newblock John Wiley \& Sons, 2014.

\bibitem[Saber \& Maillard(2024)Saber and Maillard]{saber2024bandits}
Hassan Saber and Odalric-Ambrym Maillard.
\newblock Bandits with multimodal structure.
\newblock In \emph{Reinforcement Learning Conference}, volume~1, pp.\ ~39,
  2024.

\bibitem[Saber et~al.(2021)Saber, M{\'e}nard, and Maillard]{saber2021indexed}
Hassan Saber, Pierre M{\'e}nard, and Odalric-Ambrym Maillard.
\newblock Indexed minimum empirical divergence for unimodal bandits.
\newblock \emph{Advances in Neural Information Processing Systems},
  34:\penalty0 7346--7356, 2021.

\bibitem[Saber et~al.(2024)Saber, Pesquerel, Maillard, and
  Talebi]{saber2024logarithmic}
Hassan Saber, Fabien Pesquerel, Odalric-Ambrym Maillard, and Mohammad~Sadegh
  Talebi.
\newblock Logarithmic regret in communicating mdps: Leveraging known dynamics
  with bandits.
\newblock In \emph{Asian Conference on Machine Learning}, pp.\  1167--1182.
  PMLR, 2024.

\bibitem[Sekar et~al.(2020)Sekar, Rybkin, Daniilidis, Abbeel, Hafner, and
  Pathak]{sekar2020planning}
Ramanan Sekar, Oleh Rybkin, Kostas Daniilidis, Pieter Abbeel, Danijar Hafner,
  and Deepak Pathak.
\newblock Planning to explore via self-supervised world models.
\newblock In \emph{International conference on machine learning}, pp.\
  8583--8592. PMLR, 2020.

\bibitem[Stewart \& Sun(1990)Stewart and Sun]{stewart1990matrix}
G.W. Stewart and J.~Sun.
\newblock \emph{Matrix Perturbation Theory}.
\newblock Computer Science and Scientific Computing. Elsevier Science, 1990.
\newblock ISBN 9780126702309.

\bibitem[Thompson(1933)]{thompson1933likelihood}
William~R Thompson.
\newblock On the likelihood that one unknown probability exceeds another in
  view of the evidence of two samples.
\newblock \emph{Biometrika}, 25\penalty0 (3-4):\penalty0 285--294, 1933.

\bibitem[Tu \& Recht(2018)Tu and Recht]{tu2018least}
Stephen Tu and Benjamin Recht.
\newblock Least-squares temporal difference learning for the linear quadratic
  regulator.
\newblock In \emph{International Conference on Machine Learning}, pp.\
  5005--5014. PMLR, 2018.

\bibitem[Tu \& Recht(2019)Tu and Recht]{tu2019gap}
Stephen Tu and Benjamin Recht.
\newblock The gap between model-based and model-free methods on the linear
  quadratic regulator: An asymptotic viewpoint.
\newblock In \emph{Conference on learning theory}, pp.\  3036--3083. PMLR,
  2019.

\bibitem[Zhang et~al.(2023)Zhang, Mao, Mowlavi, Benosman, and
  Ba{\c{s}}ar]{zhang2023controlgym}
Xiangyuan Zhang, Weichao Mao, Saviz Mowlavi, Mouhacine Benosman, and Tamer
  Ba{\c{s}}ar.
\newblock Controlgym: Large-scale control environments for benchmarking
  reinforcement learning algorithms.
\newblock \emph{arXiv preprint arXiv:2311.18736}, 2023.

\end{thebibliography}
\bibliographystyle{rlj}


\begin{center}
    {\beginSupplementaryMaterials} 
\end{center}

\appendix

\section{Proofs of the main propositions}
In this section, we detail the proof of Proposition~\ref{prop:elr} that provides the form of the asymptotic per-step expected log-likelihood ratio when following a given policy with control $K$.
We then detail the proof of Proposition~\ref{prop:sub}  which provides an approximation of the cost function to be optimized in the regime of small perturbations, which yields a closed-form approximate solution.

\subsection{Asymptotic per-step expected log-likelihood ratio for LQR} \label{proof-1}

\begin{proof}[Proof of Proposition \ref{prop:elr}.]
The one-step likelihood of observing $x_{t+1}$ given $x_t$ under $\Theta$ (ignoring constants) is 
$
    \mathbf{p}(x_{t+1}\vert x_t) \propto \exp \left( -\frac{1}{2}(x_{t+1} - A_Kx_t)^\intercal \Omega^{-1} (x_{t+1} - A_Kx_t) \right).
$
We denote by $\tilde{\mathbf{p}}$ the transition probability under $\widetilde{\Theta}$. Thus the one-step likelihood ratio is
\begin{equation}
\begin{split}
    \ell_t &= \log \frac{\mathbf{p}(x_{t+1}\vert x_t)}{\mathbf{\tilde{p}}(x_{t+1}\vert x_t)} \\
    &= \frac{1}{2} \left ( (x_{t+1} - \widetilde{A}_K x_t)^\intercal\Omega^{-1}(x_{t+1} - \widetilde{A}_K x_t) - (x_{t+1} - A_K x_t)^\intercal\Omega^{-1}(x_{t+1} - A_K x_t) \right ) \\
    &= \frac{1}{2} \bigg( \left ( (A_K - \widetilde{A}_K) x_t + w_t\right)^\intercal \Omega^{-1} \left ( (A_K - \widetilde{A}_K) x_t + w_t\right) - w_t^\intercal\Omega^{-1}w_t\bigg)\\
    &= \frac{1}{2} \left ( x_t^\intercal (A_K - \widetilde{A}_K)^\intercal \Omega^{-1} (A_K - \widetilde{A}_K)x_t + 2 w_t^\intercal \Omega^{-1}(A_K - \widetilde{A}_K)x_t\right ),
\end{split}
\end{equation}
taking the expectation, the second term vanishes, and we have
\begin{equation}
\begin{split}
    \mathbb{E}_{\Theta}[\ell_t]&= \frac{1}{2} \mathbb{E}_\Theta \left [ x_t^\intercal (A_K - \widetilde{A}_K)^\intercal \Omega^{-1} (A_K - \widetilde{A}_K)x_t \right] \\
    &= \frac{1}{2} \operatorname{Tr} \left( (A_K - \widetilde{A}_K)^\intercal \Omega^{-1} (A_K - \widetilde{A}_K) \Sigma_K(\Theta) \right ),
\end{split}
\end{equation}
where the stationary distribution $\Sigma_K(\Theta) = \mathbb E_\Theta \left [x_tx_t^\intercal \right| K] = \Omega + A_K \Sigma_K(\Theta) A_K^\intercal$, 
satisfies a discrete-time Lyapunov equation. For a trajectory $\tau$ of $T$ steps, the total expected log-likelihood ratio is
\begin{equation}
    \mathbb{E}_\Theta \left [ \log \frac{\mathbf{p}(\tau)}{\tilde{\mathbf{p}}(\tau)} \right ] = \sum^T_{t =1} \mathbb E_\Theta [\ell_t] = \frac{T}{2} \operatorname{Tr} \left( (A_K - \widetilde{A}_K)^\intercal \Omega^{-1} (A_K - \widetilde{A}_K) \Sigma_K(\Theta) \right ).
\end{equation}
Taking the limit as $T \rightarrow \infty$, we see that the total expected log-likelihood ratio diverges linearly, while the per-step average converges to
\begin{equation}
    {\bf d}_K(\Theta \Vert \widetilde{\Theta}) = \lim_{T \rightarrow \infty} \frac{1}{T} \mathbb{E}_\Theta \left [ \log \frac{\mathbf{p}(\tau)}{\tilde{\mathbf{p}}(\tau)} \right ] = \frac{1}{2} \operatorname{Tr} \left( (A_K - \widetilde{A}_K)^\intercal \Omega^{-1} (A_K - \widetilde{A}_K) \Sigma_K(\Theta) \right ).
\end{equation}
\end{proof}

\subsection{Sub-optimality cost refinement under small perturbations} \label{proof-2}

\begin{proof}[Proof of Proposition \ref{prop:sub}]
We begin by expressing the cost for the perturbed system $J_K(\Theta(\alpha))$, as
\begin{equation}
 \sigma_w^2\,\operatorname{Tr}\bigl(P_K(\Theta(\alpha))\bigr)
= \sigma_w^2\,\mathbf{i}^\top\operatorname{vec}\bigl(P_K(\Theta(\alpha))\bigr) = \sigma_w^2\,\mathbf{i}^\top\Bigl(I_{d^2}-A_K^\top(\alpha)\otimes A_K^\top(\alpha)\Bigr)^{-1}\mathbf{q}_K,
\end{equation}
with $\mathbf{i}= \operatorname{vec}(I_d)$ and $\mathbf{q}_K=\operatorname{vec}(Q_K)$. The closed-loop dynamics for the interpolated system is
\begin{equation}
A_K(\alpha)=A-BK+\alpha\bigl(\Delta_A+\Delta_BK\bigr)=A_K+\alpha\,\Delta_K.
\end{equation}
Its Kronecker square naturally expands as a quadratic function of $\alpha$,
\begin{equation}
A_K^\top(\alpha)\otimes A_K^\top(\alpha)
= X_K+\alpha\,\overline{X}_K+\alpha^2\,\overline{\overline{X}}_K,
\end{equation}
where $X_K=A_K^\top\otimes A_K^\top$, $\overline{X}_K = (A_K\otimes \Delta_K+\Delta_K\otimes A_K)^\intercal$ and $\overline{\overline{X}}_K = (\Delta_K\otimes \Delta_K)^\intercal$. Thus, the inverse appearing in the cost can be written in terms of perturbation, as
\begin{equation}
\Bigl(I_{d^2}-A_K^\top(\alpha)\otimes A_K^\top(\alpha)\Bigr)^{-1}
=\Bigl(I_{d^2}-X_K-\widetilde{X}_K(\alpha)\Bigr)^{-1},
\end{equation}
with $\widetilde{X}_K(\alpha)=\alpha\,\overline{X}_K+\alpha^2\,\overline{\overline{X}}_K$. Assuming that the perturbations are small, we apply a first-order expansion of the infinite series, as described in Section 2.2.4 of \cite{stewart1990matrix}, to obtain
\begin{equation}
\Bigl(I_{d^2}-X_K-\widetilde{X}_K(\alpha)\Bigr)^{-1}\approx Y_K-Y_K\,\widetilde{X}_K(\alpha)\,Y_K,
\end{equation}
where $Y_K=(I_{d^2}-X_K)^{-1}$. For clarity, we introduce the scalar coefficients $p_K=\mathbf{i}^\top Y_K\,\mathbf{q}_K$, $\overline{p}_K=\mathbf{i}^\top Y_K\,\overline{X}_K\,Y_K\,\mathbf{q}_K$, and $\overline{\overline{p}}_K=\mathbf{i}^\top Y_K\,\overline{\overline{X}}_K\,Y_K\,\mathbf{q}_K$.
Hence, the cost function is simplified to
\begin{equation}
\mathbf{i}^\top\Bigl(I_{d^2}-A_K^\top(\alpha)\otimes A_K^\top(\alpha)\Bigr)^{-1}\mathbf{q}_K
\approx p_K-\alpha\,\overline{p}_K+\alpha^2\,\overline{\overline{p}}_K.
\end{equation}

Repeating the derivation for another gain $K'$ and equating the two expressions for $\mathcal{L}(\alpha)$ leads to
\begin{equation}
\begin{gathered}
(p_K-p_{K'})-\alpha\,(\overline{p}_K-\overline{p}_{K'})+\alpha^2\,(\overline{\overline{p}}_K-\overline{\overline{p}}_{K'})=0, \\
\boxed{
\alpha=\frac{(\overline{p}_K-\overline{p}_{K'})\pm\sqrt{(\overline{p}_K-\overline{p}_{K'})^2-4(\overline{\overline{p}}_K-\overline{\overline{p}}_{K'})(p_K-p_{K'})}}
{2(\overline{\overline{p}}_K-\overline{\overline{p}}_{K'})}.
}
\end{gathered}
\end{equation}
Choosing the positive solution completes the derivation. Finally, using the identities $
    \operatorname{vec}(AXB) = (B^\intercal\otimes A)\operatorname{vec}(X)$, and
    $\operatorname{vec}(I_d)^\intercal \operatorname{vec}(X) = \operatorname{Tr}(X),
$ and the Neumann series expansion,
Kronecker products and vectorizations simplify and complete the proof.
\end{proof}

\end{document}